\documentclass{amsart}
\usepackage{amssymb,amsmath,amsfonts,amscd,amsthm}
\usepackage{graphicx}
\usepackage{algorithmic}
\usepackage{algorithm}
\usepackage{subfig}
\usepackage{caption}
\usepackage{url}

\newtheorem{theorem}{Theorem}[section]

\newtheorem{corollary}[theorem]{Corollary}
\newtheorem{lemma}[theorem]{Lemma}

\newtheorem{assumption}[theorem]{Assumption}

\newcommand{\R}{\mathbb{R}}

\newcommand{\C}{\mathbb{C}}

\newcommand{\N}{\mathbb{N}}

\newcommand{\Sb}{\mathbb{S}}

\newcommand{\HH}{\mathcal{H}}

\newcommand{\FF}{\mathcal{F}}

\newcommand{\mb}{\mathbf}

\DeclareMathOperator{\vol}{Vol}

\DeclareMathOperator{\dist}{dist}

\begin{document}
\title{Spectral Convergence Rate of Graph Laplacian}

\author{Xu~Wang\\\\
       Dept.\ of Mathematics, \\ University of California, San Diego, CA
       92093, USA\\ 
       xuw014@ucsd.edu
       }

\maketitle

\begin{abstract}
Laplacian Eigenvectors of the graph constructed from a data set are used in many spectral manifold learning algorithms such as diffusion maps and spectral clustering. Given a graph constructed from a random sample of a $d$-dimensional compact submanifold $M$ in $\R^D$, we establish the spectral convergence rate of the graph Laplacian. It implies the consistency of the spectral clustering algorithm via a standard perturbation argument. A simple numerical study indicates the necessity of a denoising step before applying spectral algorithms.
\end{abstract}

\section{Introduction}

High-dimensional data appears naturally in real-world applications. A common assumption is that the data resides on a low-dimensional manifold. Since the underlying manifold is usually unknown, Belkin and Niyogi~\cite{DBLP:conf/nips/BelkinN01} proposed the framework to use the graph Laplacian to approximate the Laplace operator on such a manifold. Many manifold learning algorithms such as diffusion maps~\cite{MR2206766}, Hessian eigenmaps~\cite{Donoho13052003}, spectral clustering~\cite{DBLP:conf/nips/NgJW01} are based on the eigenvectors of graph Laplacians.

It is desirable to understand the approximation quality of the graph Laplacian from finite samples to the Laplacian operator on the underlying manifold. Indeed, there has been a flurry of works investigating different types of convergence for the graph Laplacian. For example, the pointwise convergence has been shown in~\cite{DBLP:conf/colt/HeinAL05, MR2238670, MR2460286}. Regarding the spectral convergence, Belkin and Niyogi~\cite{DBLP:conf/nips/BelkinN06} showed the convergence of eigenvectors under the assumption that the dataset is sampled from the uniform distribution on a manifold without noise. Luxburg et al.~\cite{MR2396807} studied the graph Laplacian constructed by kernels with fixed bandwidth $h$ and provided spectral convergence rate analysis under a general probability model. We note that these spectral convergence results are more relevant to algorithms related to Laplacian eigenvectors than the pointwise convergence results.

Unfortunately, neither~\cite{DBLP:conf/nips/BelkinN06} or~\cite{MR2396807} provided a clear picture as for how the bandwidth $h$ should be scaled w.r.t. the data size $n$. As a consequence, the choice of $h$ is still unprincipled. More recently, Trillos and Slep$\check{c}$ev~\cite{Nicolas15} studied the convergence rate with varying bandwidth $h$ in the case when dataset is sampled from an open, bounded, connected set with Lipschitz boundary $S\subset \R^D$. It is claimed in~\cite{Nicolas15} that their convergence rate is (almost) optimal, depending on the ambient dimension $D$.

In this paper, we consider the model of $d$-dimensional manifolds in a high dimension $\R^D$ ($d\ll D$), which is not included in the analysis of~\cite{Nicolas15} (since it is not an open set). It is expected that the convergence rate should depend on the intrinsic dimension $d$ rather than the ambient dimension $D$. The goal of this paper is to analyze the convergence rate of graph Laplacians w.r.t. varying bandwidth $h$ under this manifold modeling. 

We show that the convergence rate indeed depends only on $d$ (Theorem~\ref{thm:cr} and Corollary~\ref{cor:cr}) in the noise-free case. When noise is added so that points can lie off the manifold, the convergence rate depends on the ambient dimension $D$. Indeed, we show in the numerical study that the choice of the bandwidth $h$ is heavily influenced by the ambient dimension. This observation indicates that denoising should be a necessary step before applying spectral algorithms in order to achieve better convergence rate (ideally independent of $D$). In this paper, we answer the question how the kernel bandwidth $h$ and $n$ should be scaled when taking limits in both $h$ and $n$ simultaneously. We then apply the results to the spectral clustering algorithm to obtain its consistency (Corollary~\ref{cor:cs}).

\section{Background}

\subsection{Graph Laplacian}
Given a dataset $X=\{x_i\}_{i=1}^n$ i.i.d.~sampled from a compact manifold $M$ with a prbability measure $P$ and a similarity matrix $K=\left( k_{ij} \right)_{i,j=1,\ldots, n}$, we define the diagonal degree matrix $D$ with the $i$-th diagonal entry $d_i = \sum_{x_j \in X} k_{ij}$ and the symmetric normalized graph Laplacian as
\begin{equation}\label{eq:deflap}
L = I - D^{-1/2} K D^{-1/2}.
\end{equation}
To talk about the convergence of the $n$-by-$n$ matrix $L$ as $n\rightarrow \infty$, we need to work in a proper function space. We consider the $L_2$-space $L_2(P)$ over $M$ and construct relevant continuous linear operators on it. First of all, we define the degree function as
\begin{equation}\label{eq:defdeg}
d_h(x) = \int_M k_h(x,y) d P(y),
\end{equation}
which describes the local density at the scale $h$. Then we define two linear operators on $L^2(P)$ as follows:
\begin{equation}\label{eq:defT}
\begin{aligned}
T_h f &= \int_M k_h(x,y)(d_h(x)d_h(y))^{-1/2} f(y) d P(y), \\
T_{h}^u f &= d_h(x)^{-1} \int_M k_h(x,y) f(y) d P(y).
\end{aligned}
\end{equation}
The first operator is the symmetric normalized graph Laplacian. The second operator is the random walk normalized graph Laplacian. To make a connection between the above linear operator $T_h$ with $L$, we further define a linear operator $T_{n,h}$ on $L^2(P)$ based on the dataset $X$. It is obtained by replacing $P$ in~\eqref{eq:defdeg} and~\eqref{eq:defT} with the empirical measure. Specifically,
\begin{equation}\label{eq:defnT}
T_{n,h} f = \sum_{y_i\in X} k_h(x,y_i)(d_{n,h}(x)d_{n,h}(y_i))^{-1/2} f(y_i),
\end{equation}
with $d_{n,h}(x) = \sum_{y_i \in X} k_h(x,y_i)$. 

It is shown in Proposition 9 of ~\cite{MR2396807} that the spectrum of $L$ and $I - T_{n,h}$ are more or less the same and the eigenvectors of $L$ is the restriction of eigenfunctions of $I - T_{n,h}$ on the dataset $X$. Indeed, if we make the following identification:
$$
f\in L^2(M) \mapsto \left(f(x_1), \ldots, f(x_n)\right)^T\in \R^n \text{ for } x_i\in X.
$$
and define $\hat{T}_{n,h}$ on $\R^n$ to be the matrix with entries
$$
(\hat{T}_{n,h})_{ij} = k_h(x_i, x_j)(d_{n,h}(x_i)d_{n,h}(x_j))^{-1/2},
$$ 
then $L = I - \hat{T}_{n,h}$. To show the convergence of $L$, it is enough to show the spectral convergence of $T_{n,h}$ to $T_h$ over $L_2(P)$. 

Luxburg et al.~\cite{MR2396807} analyzed the spectral convergence under the assumption that $k(x,y)$ is a symmetric, continuous and strictly positive function. They showed the convergence to some limiting object but no justification regarding whether such a limit itself produces the ``correct'' result for spectral clustering. The main constraint is that their analysis did not allow varying bandwidth $h$. In this work, we assume that $k_h(x, y)$ is the Gaussian kernel with varying bandwidth $h$. We provide more detailed analysis on the effect of $h$ w.r.t. the spectral convergence rate. It sheds light on how $h$ is scaled with $n$ when taking limit of both $h$ and $n$. It also turns out to be crucial in justifying the spectral clustering via a perturbation argument.

\subsection{Spectral Clustering}

This section briefly reviews the spectral clustering algorithm. A detailed introduction of spectral clustering can be found in~\cite{luxburg_tutorial}. Spectral clustering can be summarized as follows:
\begin{algorithm}[H]
\caption{Spectral clustering with symmetric normalized graph Laplacian}
\label{algm:sc}
\begin{algorithmic}
\REQUIRE $X=\{x_1, x_2,\ldots, x_n\}\in \R^d$: data, $m$: number of clusters, $h$: scaling parameter for kernel.
\ENSURE Index set $\{g_i\}_{i=1}^n$ for $m$ partitions such that $x_i\in X$ belongs to group with label $g_i$.\\
\textbf{Steps}:\\
$\bullet$ Find the smallest $m$ eigenvectors $\{\hat{\mb{v}}_i\}_{i=1}^m$ for $L$ (see~\eqref{eq:deflap}). \\
$\bullet$ Normalize each row of $[\hat{\mb{v}}_2,\cdots, \hat{\mb{v}}_m]$ and call $K$-means on the rows.

\RETURN A cluster label $g_i$ for each point $x_i$.

\end{algorithmic}
\end{algorithm}

The intuition behind Algorithm~\ref{algm:sc} is that if $(k_{ij})_{i,j=1}^n$ defines a graph with multiple components, each of which corresponds to a cluster, then the null space of $L$ is spanned by indicator vectors of these clusters. Unfortunately, there is no selection of kernels that is guaranteed to satisfy this condition unless the clusters are known beforehand. Luxburg et al.~\cite{luxburg_tutorial} mentioned to apply a perturbation argument (via Davis-Kahn theorem) to cases when the graph ``almost'' satisfies this condition, this idea is not carried out rigorously. As an application of our spectral convergence rate result, we work out the details of this perturbation argument and identify the hidden assumptions under which it is valid. Indeed, we show that the perturbation argument works when $k_h(x,y)$ has exponential decay. Taking the statistical point of view, we show that the partition of data converges to some limiting partition of the domain, from which the data is i.i.d.~sampled.

\subsection{Notation}

In this section, we introduce more notation that is used later. Let $(M, \dist_M, P)$ be a Riemannian manifold with a probability $P$. Here $\dist_M$ denotes the intrinsic metric on $M$. For a function $f$ over $M$ and a dataset $X = \{x_i\}_{i=1}^n$ i.i.d.~sampled from $P$, we denote the expectation $\int_{M} f dP$ and the empirical expectation $\displaystyle \sum_{x_i\in X} f(x_i)$ by $Pf$ and $P_n f$ respectively. Let $N(M, \epsilon, \dist_M)$ be the $\epsilon$-covering number of $M$. Let $k_h(\mb{x}, \mb{y}) = \exp(-\|\mb{x}- \mb{y}\|^2/h^2)$. Let $u(\cdot)$ be a function over $M$. We define four function spaces $\mathcal{K}_h$, $\HH_h$, $u\HH_h$ and $\HH_h \HH_h$ over $M$ as follows:
\begin{equation}
\begin{aligned}
& \mathcal{K}_h = \{k_h(x, \cdot) | x\in M\}, \HH_h = \{k_h(x,\cdot)(d_h(x)d_h(\cdot))^{-1/2}| x\in M\}, \\
& u\HH_h = \{u(\cdot)k_h(x,\cdot)(d_h(x)d_h(\cdot))^{-1/2}| x\in M\}, \\
& \HH_h\HH_h = \{k_h(x,\cdot)(d_h(x)d_h(\cdot))^{-1/2}k_h(y,\cdot)(d_h(y)d_h(\cdot))^{-1/2}| x,y\in M\}.
\end{aligned}
\end{equation}
$k_h(\mb{x}, \mb{y})$ is equivalent to the standard Gaussian kernel up to a constant factor. Since the operators are normalized by the degree function (with the constant factor being cancelled), the resulting operators are exactly the same. $C(M)$ denotes the set of continuous functions over $M$. We use $\|\cdot\|$ to denote the $L_{\infty}$-norm. Let $\Delta_M$ ($\Delta_P$) be the (weighted) Laplacian operator on $M$ with discrete spectrum $0=\lambda_1 < \lambda_2 \leq \cdots$. Let $\lambda_h$ be the $i$-th largest eigenvalue of $T_h$ with the eigenfunction $u_h$ s.t. $\|u_h\| = 1$. Let $\lambda_{n,h}$ be the $i$-th largest eigenvalue of $T_{n,h}$ and $u_{n,h}$ be its eigenfunction s.t. $\|u_{n,h}\| = 1$. We also introduce an operator lying between $T_h$ and $T_{n,h}$, 
$$
\bar{T}_{n,h} = \sum_{y_i\in X} k_h(x,y_i)(d_{h}(x)d_{h}(y_i))^{-1/2} f(y_i).
$$
We note that $\bar{T}_{n,h}$ differs from $T_{n,h}$ by replacing the empirical degree function $d_{n,h}$ with the degree function $d_h$. Let $\sigma(T)$ be the spectrum of the linear operator $T$. Let $R_z(T) = (T - zI)^{-1}$ be the resolvent of $T$. 

\section{Main results}

\subsection{Spectral Convergence}

Theorem~\ref{thm:cr} provides the spectral convergence rate of the graph Laplacian $T_{n,h}$ to $T_h$ (defined in ~\eqref{eq:defnT} and~\eqref{eq:defT}). When proving Theorem~\ref{thm:cr}, we need the following assumptions on the underlying manifold $M$ and the probability measure $P$.

\begin{assumption}\label{assum1} 
(1) $M$ has a bounded diameter $D_{M}$ (in terms of the Riemannian metric);
(2) $P$ has support on $M$ and a continuous density function $p$ with positive lower bound $c_p>0$;
(3) $\lambda_i$ is a simple eigenvalue of $\Delta_P$.
\end{assumption}

Under the above assumptions, the $i$-th largest eigenvalue $\lambda_{n,h}$ and its eigenfunction $u_{n,h}$ of $T_{n,h}$ converge to those of $T_h$ in Theorem~\ref{thm:cr}.

\begin{theorem}\label{thm:cr}
Let $(M, P)$ be a $d$-dimensional compact Riemannian manifold in $\R^D$ with a probability measure $P$. Let $X$ be a dataset of size $n$, i.i.d.~sampled from $P$. Let $T_{n,h}$ and $T_h$ be constructed by using $k_h(\mb{x}, \mb{y}) = \exp(-\|\mb{x}- \mb{y}\|^2/h^2)$. If $\sqrt{n}h^{4d+3}\geq C$ and $h\leq h_0$, then with probability at least $1-\delta$,
$$|\lambda_{n,h} - \lambda_{h}| \leq \frac{C'}{\sqrt{n}h^{5d+3}}, \quad\|a_n u_{n,h} - u_h\| \leq \frac{C''}{\sqrt{n}h^{4d+3}},$$
for some constants $C, C', C''>0$ and a sequence $a_n \in \{-1, 1\}$. In particular, if $\sqrt{n}h^{5d+3}\rightarrow \infty$ and $h\leq h_0$, then
$$
|\lambda_{n,h} - \lambda_{h}| \rightarrow 0, \quad \|a_n u_{n,h} - u_h\| \rightarrow 0.
$$
with probability at least $1-\delta$.
\end{theorem}

The constants $C, C', C''$ depends on the Riemannian manifold $M$ and the probability $P$. In the proof of Theorem~\ref{thm:cr}, all constants are clearly defined. We summarize the dependence and references of these constants in the proof for the readers' convenience as follows.

\begin{itemize}

\item $C_1$ is defined in Lemma~\ref{lm:empbd}. It depends on the dimension $d$, the failure probability $\delta$.
\item $C_2, C_{\delta}$ are defined in Lemma~\ref{lm:dg}. They depend on $C_1$ and the lower bound $c_p$ of the probability density and the curvature of $M$.
\item $C_{\sigma}$ and $h_0$ are defined in Lemma~\ref{lm:eigen-gap}. They depend on the eigen-gap between $\lambda_i$ and the rest spectrum of the Laplacian operator on $M$. 
\item $C = \max \left( \left( \frac{3}{C_{\delta}^4}+1 \right) \frac{4C_1}{C_{\sigma}}, C_2 \right)$ is defined in Lemma~\ref{lm:numerical}. $C' = \frac{C_1}{C_{\delta}^3} +C_1 +\frac{2C''}{C_{\delta}}$ is defined in~\eqref{eq:nlbd} and $C'' = \frac{8\pi C_1}{C_{\sigma}} \left( \frac{1}{C_{\delta}^3}  + \frac{3}{C_{\delta}^4}+2  \right) $ is defined in~\eqref{thm:convvec1}.

\end{itemize}

When $P$ is the uniform distribution on $M$, Belkin and Niyogi~\cite{DBLP:conf/nips/BelkinN06} studied the convergence rate of $\frac{1-T_h}{h^2}$ to the Laplacian operator $\Delta_M$. Combining this result with Theorem~\ref{thm:cr}, we obtain the convergence rate of $T_{n,h}$ to the Laplacian operator $\Delta_M$ in Corollary~\ref{cor:cr}.

\begin{corollary}\label{cor:cr}
In addition to the same assumptions as in Theorem~\ref{thm:cr}, we assume furthermore $P$ is the uniform distribution on $M$. Then the convergence rate of $\frac{1-\lambda_{n,h}}{h^2}$ and $u_{n,h}$ to $\lambda_i$ and $u_i$ respectively is $O(n^{-2/(5d+6)(d+6)})$. It is achieved by Picking $h = n^{-1/(10d+12)}$.
\end{corollary}

There are two factors (both depending linearly on $d$) that contribute to the quadratic dependence on $d$ for the convergence rate from $T_{n,h}$ to $\Delta_M$ (see Section~\ref{sec:cor:cr}). The first factor is the convergence from $T_{n,h}$ to $T_h$. The second factor is the convergence from $T_h$ to $\Delta_M$. It is still unknown whether the dependence on $d$ could be linear. One possibility is to improve the spectral convergence from $T_h$ to $\Delta_M$ since the pointwise convergence rate~\cite{MR2238670} is known to be independent of $d$. 

\subsection{Spectral Clustering: Two-region Models}

One advantage of spectral clustering is the ability of clustering regions with complex geometry. We first specify the multi-manifold model under which the consistency is established. Suppose $M = \displaystyle \cup_{i=1}^K M_i$ be the union of $K$ connected regions of dimension $d$ and the probability measures $P_i$ ($i=1,\ldots, K$) contained in the unit ball of $\R^D$. Let $b_1,\ldots, b_K$ be positive numbers s.t. $\sum_{i=1}^K b_i=1$. We assume $X$ is i.i.d.~sampled from $P = \sum_{i=1}^K b_i P_i$. We analyze the spectral clustering in the case when $K=2$. The cases for $K>2$ can be analyzed similarly. In addition to Assumption~\ref{assum1} on $\{(M_i, P_i)\}_{i=1}^2$, the following assumptions are made in Corollary~\ref{cor:cs}. 

\begin{assumption}\label{assum2}
(1) $M_1$ and $M_2$ are separated by a distance $c_d$. In other words, $\min_{x_1\in M_1,x_2\in M_2}dist(x_1, x_2) = c_d >0$.
(2) the null space of $\Delta_{P_i}$ is $1$-dimensional for $i=1,2$.
\end{assumption}

The consistency of spectral clustering in Corollary~\ref{cor:cs} is a direct consequence of the spectral convergence of graph Laplacian and a stability result on $K$-means~\cite{Nicolas15} in the last step of spectral clustering.

\begin{corollary}\label{cor:cs}
The spectral clustering is consistent in the multi-manifold modeling when $h\rightarrow 0$, $\sqrt{n}h^{5d+3}\rightarrow \infty$. That is, the algorithm can correctly identify the region for each submanifold and cluster the data points correctly with probability approaching $1$ as $n\rightarrow \infty$.
\end{corollary}

\section{Proof of Theorem~\ref{thm:cr}: The Noise-Free Case}\label{sec:embed}

In this section, we fix one manifold $(M, P)$ with a probability measure on it. Let $\Delta_P$ be the weighted Laplacian operator on $M$ with the discrete spectrum $0=\lambda_1 < \lambda_2 < \cdots$. Fixing $i$, we assume $\lambda_i$ is a simple eigenvalue with an eigenfunction $u_i$. Let $\lambda_h$ be the $i$-th largest eigenvalue of $T_h$ with an eigenfunction $u_h$. Let $\lambda_{n,h}$ be the corresponding eigenvalue of $T_{n,h}$ and $u_{n,h}$ be its eigenfunction. WLOG, we assume $\|u_h\| = \|u_{n,h}\|= 1$.

The proof of Theorem~\ref{thm:cr} contains four steps. The first step is to bound the error of empirical expectations over the function set $\FF_h = \mathcal{K}_h \cup (u_h \HH_h) \cup (\HH_h\HH_h)$. The second step is to bound relevant operators by this empirical expectation error bound. The third step is to bound the error of eigenfunctions from the bounds of these operators. The last step is to bound the eigenvalues from the bound of eigenfunctions.

\subsection{Step I: Error Bound of The Empirical Expectations}

We establish an uniform upperbound of the error between the expectation and the empirical expectation over the set of functions, $\FF_h = \mathcal{K}_h \cup (u_h \HH_h) \cup (\HH_h\HH_h)$. To begin with, we state a technical lemma about the covering number of $\mathcal{K}_h$.

\begin{lemma}\label{lm:cn}
The covering number $N(\mathcal{K}_h, \epsilon, \|\cdot\|) \leq  \left( \frac{24\sqrt{2d}D_M}{\epsilon h^2} \right)^{2d}$.
\end{lemma}

\begin{proof}
We derive the covering number of $\mathcal{K}_h$ from that of $M$. Lemma~2.4 in~\cite{MRBM15} states that
$$
N(M, \epsilon_1, \dist_{\R^D}) \leq N(M, \epsilon_1, \dist_M) \leq \left( \frac{6\sqrt{2d}D_{M}}{\epsilon_1} \right)^{2d}.
$$
We fix an $\epsilon_1$-net $\{x_i\}_{i=1}^m$ of $M$. For any $x\in M$, there exists $x_i$ s.t. $\dist_{\R^D}(x, x_i) \leq \epsilon_1$. Then
\begin{equation}
\begin{aligned}
\displaystyle \max_{y\in M}|k_h(x,y) - k_h(x_i,y)| &\leq \max_{y\in M}\max_{t_0\in [0,1]} |(k_h)'_x(x_i +t_0 (x-x_i),y)\cdot (x-x_i) | \\
&\leq \max_{y\in M}\max_{t_0\in [0,1]} \|(k_h)'_x(x_i +t_0 (x-x_i),y)\| \epsilon_1 \leq \frac{4\epsilon_1}{h^2}.
\end{aligned}
\end{equation}
If we take $\epsilon_1 = \epsilon h^2/4$, then $\{k_h(x_i, \cdot)\}_{i=1}^m$ is an $\epsilon$-net of $\mathcal{K}_h$. Thus,
$$
N(\mathcal{K}_h, \epsilon, \|\cdot\|) \leq m \leq \left( \frac{24\sqrt{2d} D_M}{\epsilon h^2} \right)^{2d}.
$$

\end{proof}

Lemma~\ref{lm:empbd} states a uniform upperbound of the error between the empirical expectation and the expection over the function set $\FF_h$.
\begin{lemma}\label{lm:empbd}
Let $(M, \dist_M, P)$ be a probability space with a metric, $\FF_h$ be the class of real-valued functions defined above with $h\leq \sqrt{24\sqrt{2d} D_M}$. Let $(X_n)_{n \in \N}$ be a sequence of i.i.d.~random variables drawn according to $P$, and $(P_n)_{n\in\N}$ the corresponding empirical distributions. Then there exists some constant $C_1>0$ such that, for all $n\in\N$ with probability at least $1-\delta$,
\begin{equation}\label{eq:empiricalbound}
\sup_{f\in \FF_h} |P_n f -P f| \leq \frac{C_1}{\sqrt{n}h}.
\end{equation}
\end{lemma}

\begin{proof}

We note the following entropy bound from Theorem~19 and Proposition~20 in~\cite{MR2396807}, for a constant $c_1>0$,
\begin{equation}\label{lm:empbd:eq1}
\sup_{f\in \FF_h} |P_n f -P f| \leq \frac{c_1}{\sqrt{n}} \int_0^{\infty} \sqrt{\ln N(\FF_h, \epsilon, L_2(P_n))} d\epsilon + \sqrt{\frac{1}{2n}\log \frac{2}{\delta}}.
\end{equation}

Let $l = \min_{x,y\in M} k_h(x, y)$, $s=\frac{\|k_h\|+2\sqrt{l\|k_h\|}}{2l^2}$ and $q = \min\{1, \|u_h\|s, \frac{\|k\|}{l}s\}$. From Proposition~20 of~\cite{MR2396807}, we know that 
\begin{equation}\label{lm:empbd:eq2}
N(\FF_h, \epsilon, \|\cdot\|) \leq 3 N(\mathcal{K}_h, q\epsilon, \|\cdot\|),
\end{equation}
Moreover, Lemma~\ref{lm:cn} implies that when $\epsilon \leq c_2:=\frac{h^2}{24\sqrt{2d} D_M}$ 
$$
\ln N(\mathcal{K}_h, \epsilon, \|\cdot\|) \leq 4d\ln\frac{1}{\epsilon}.
$$
Then we have 
\begin{equation}\label{lm:empbd:eq3}
\begin{aligned}
\int_0^{\infty} &\sqrt{\ln N(\FF_h, \epsilon, L_2(P_n))} d\epsilon \\
& \leq \int_0^{2/q} \sqrt{3\ln N(\mathcal{K}_h, q\epsilon, \|\cdot\|)} d\epsilon \\
& \leq \frac{\sqrt{12d}}{q} \int_0^{c_2} \left(\ln \frac{1}{\epsilon}\right)^{1/2} d\epsilon + \frac{1}{q}\int_{c_2}^{2}\sqrt{3\ln N(\mathcal{K}_h, \epsilon, \|\cdot\|)} d\epsilon \\
& \leq \frac{2\sqrt{12d}}{q} c_2\left( \ln \frac{1}{c_2} \right)^{1/2} +\frac{(2-c_2)\sqrt{12d}}{q}\left( \ln \frac{1}{c_2} \right)^{1/2} \\
& \leq \frac{c_3}{h}\quad \forall h \leq \sqrt{24\sqrt{2d} D_M}.
\end{aligned}
\end{equation}

The first inequality follows from~\eqref{lm:empbd:eq2} and the fact that $\mathcal{K}_h$ has diameter $2$ under $\|\cdot\|$. The third inequality can be easily deduced from integration by part and $c_4$ depends only on $D$. The fourth inequality is obtained by plugging up the expression of $c_2$ and noting that $\|u_h\| = 1$, $c_2\leq 1$ and $q \geq 1/2$. By taking $C_1=c_1 c_3 + \sqrt{\frac{1}{2}\log \frac{2}{\delta}}$, \eqref{eq:empiricalbound} follows readily from~\eqref{lm:empbd:eq1} and~\eqref{lm:empbd:eq3}.

\end{proof}


\subsection{Step II: Bounds of Relevant Operators}

We first provide a lower bound of the degree functions $d_h(x)$ and $d_{n,h}(x)$ and upper bounds of $\|T_{n,h}\|$ and $\|T_h\|$. 

\begin{lemma}\label{lm:dg}
There are constants $C_2, C_{\delta}>0$ such that with probability at least $1-\delta$
\begin{equation}\label{lm:dg1}
\min_{x\in M} \min\{d_h(x), d_{n,h}(x)\} \geq C_{\delta} h^{d}, \|T_{n,h}\|\leq \frac{1}{C_{\delta}h^d}, \|T_h\| \leq \frac{1}{C_{\delta}h^d}, \|\bar{T}_{n,h}\|\leq \frac{1}{C_{\delta}h^d}
\end{equation}
for $\sqrt{n}h^{d+1}\geq C_2$.
\end{lemma}

\begin{proof}
Recall that the density $p(x)$ of $P$ has a lower bound $c_p$. For sufficiently small $h$,
\begin{equation}\label{lm:dg2}
\begin{aligned}
d_h(x) = \int_M k_h(x,y) p(y) dy \geq c_p \int_M k_h(x,y) dy \geq c_p\int_{M\cap B(x, h)} k_h(x,y) dy \\
\geq c_p\vol(\R^d\cap B(x, h))(1+O(h^2))\min_{y\in M\cap B(x, h)}k_h(x,y) \geq C_p h^{d}.
\end{aligned}
\end{equation}
The third inequality in~\eqref{lm:dg2} follows from the fact $M$ can be locally approximated by $\R^d$ up to order 2.

Since $k_h(x,\cdot) \in \FF_h$, inequality~\eqref{eq:empiricalbound} implies that
$$
\sup_{x\in M} |d_{n,h}(x) - d_h(x)| \leq \frac{C_1}{\sqrt{n}h},
$$
with probability at least $1-\delta$. If we take $C_{\delta} =  C_p - C_1/C_2$, then
\begin{equation}\label{lm:dg3}
d_{n,h}(x) \geq C_p h^{d}-\frac{C_1}{\sqrt{n}h} \geq C_{\delta} h^{d}.
\end{equation}
We note that $C_{\delta}= C_p/2$ when we pick $C_2 = 2C_1/ C_p$. The inequalities for $\|T_{n,h}\|$ and $\|T_h\|$ follows immediately from definition, \eqref{lm:dg3} and $\|k_h\| \leq 1$.
\end{proof}

Next we introduce two lemmas with technical bounds for operators to be used in bounding the eigenfunctions.
\begin{lemma}\label{lm:nT}
Assume the general conditions are satisfied and $u$ is a continuous function. Then the following bounds hold:
\begin{equation}\label{lm:nT1}
\begin{aligned}
\| T_{n,h} - \bar{T}_{n,h} \| &\leq \frac{1}{C_{\delta}^3}h^{-3d} \sup_{f\in \mathcal{K}_h} |P_n f - P f|,\\
\| (\bar{T}_{n,h} - T_h) u \| &\leq \sup_{f\in u\cdot \HH_h} |P_n f - P f|,\\
\| (T_h - \bar{T}_{n,h}) \bar{T}_{n,h} \| &\leq \sup_{f\in \HH_h \cdot \HH_h} |P_n f - P f|.
\end{aligned}
\end{equation}
\end{lemma}

The above lemma can be proved similarly as in Proposition~17 in~\cite{MR2396807}. Indeed, the first inequality in~\eqref{lm:nT1} follows from Proposition~12 in~\cite{MR2396807} by using  the lower bound on degree functions in Lemma~\ref{lm:dg} and $\|k_h\| \leq 1$. The second inequality is directly from the definition. The third inequality follows from the Fubini's theorem.

\begin{lemma}\label{lm:normT}
There are constants $C_2, C_{\delta}>0$ such that with probability at least $1-\delta$
\begin{equation}\label{lm:normT1}
\|(T_h - T_{n,h}) T_{n,h}\| \leq \left( \frac{3}{C_{\delta}^4}+1 \right)h^{-4d} \sup_{f\in \FF_h} |P_n f -P f|
\end{equation}
for $\sqrt{n}h^{d+1}\geq C_2$.
\end{lemma}

\begin{proof}

\begin{equation}\label{lm:normT2}
\begin{aligned}
\|(T_h - T_{n,h}) T_{n,h}\| &\leq \|T_h\| \|\bar{T}_{n,h} - T_{n,h}\| +\|(T_h - \bar{T}_{n,h}) \bar{T}_{n,h}\| \\
& \quad+ \|\bar{T}_{n,h} \bar{T}_{n,h} - \bar{T}_{n,h} T_{n,h}\| + \| \bar{T}_{n,h} T_{n,h} - T_{n,h} T_{n,h} \| \\
& \leq \frac{3}{C_{\delta}} h^{-d} \| \bar{T}_{n,h} - T_{n,h} \| + \| (T_h - \bar{T}_{n,h}) \bar{T}_{n,h} \| \\
& \leq (\frac{3}{C_{\delta}^4} +1) h^{-4d} \sup_{f\in \FF_h} |P_n f -P f|.
\end{aligned}
\end{equation}

The first inequality in~\eqref{lm:normT2} is from the triangle inequality. The second inequality uses the the upperbounds of operators in Lemma~\ref{lm:dg}. The last inequality follows from Lemma~\ref{lm:nT}.
\end{proof}

Let $\mathrm{Pr}_{u_{n,h}}$ be the projection on $u_{n,h}$. Proposition~18 in~\cite{MR2396807} states that the error of eigenfunctions can be upper bounded by $\| u_h - \mathrm{Pr}_{u_{n,h}} u_h \|$. In the following, we bound this quantity by expressing the projection operator as an integration over the resolvent of $T_{n,h}$. We proceed with two technical lemmas, which are needed to bound the resolvent. 

\begin{lemma}\label{lm:eigen-gap}
There exist positive constants $h_0, C_{\sigma}$ such that if $h\leq h_0$ then $\lvert \lambda_h \rvert \geq \frac{1}{2}+c_0 h^2$ and $\dist(\lambda_h, \sigma(T_h)\backslash \lambda_{h} ) >2C_{\sigma} h^2$. 
\end{lemma}

\begin{proof}

Let $\lambda_h^{u}$ be the $i$-th largest eigenvalue of random walk normalized Laplacian $T_{h}^u$. Theorem~5.5 in~\cite{Amit14} shows that the $i$-th smallest eigenvalue of $\frac{I - T_{h}^u}{h^2}$ converges to $\lambda_i$ for any $i\in \N$. Let $\dist(\lambda_i, \sigma(\Delta_P)\backslash \{\lambda_i\}) >2C_{\sigma}$. Then for some positive constants $h_0>0$,
$$
\lvert \lambda_h^{u} \rvert \geq \frac{1}{2}+C_{\sigma} h^2 \text{ and }\dist(\lambda_h^u, \sigma(T_{h}^u)\backslash \lambda_{h}^{u} ) >2C_{\sigma} h^2, \quad \forall h\leq h_0.
$$
The conclusion for $\lambda_h$ follows from the fact that $T_h$ has exactly the same spectrum as $T_{h}^u$.

\end{proof}

\begin{lemma}\label{lm:bd2}
Let $\lambda$ be an isolated eigenvalue of $T$ and $\Gamma_{r}$ be the circle centered at $\lambda$ with radius $r<\dist(\lambda, \sigma(T)\backslash \{ \lambda \})/2$ in the complex plane $\C$, then 
$$
\max_{z\in \Gamma_{r}} \| R_z(T)\| \leq 1/r.
$$
\end{lemma}

\begin{proof}
Fix any point $z \in \Gamma_{r}$. Let $\lambda_{r}$ be the eigenvalue of $R_z(T)$ with the largest magnitude. Then $\|R_z(T)\| = |\lambda_{r}|$. We show that $|\lambda_{r}| \leq 1/r$. Let $v$ be a vector such that $R_z(T) v = \lambda_{r} v$. Then
$$
T v = (z -\frac{1}{\lambda_{r}}) v.
$$
This means that $r = \dist(z, \sigma(T)) \leq 1/|\lambda_{r}|$. Thus, $\|R_z(T)\| = |\lambda_{r}| \leq 1/r$.
\end{proof}

With the above two lemmas, we now bound $\|u_h - \mathrm{Pr}_{u_{n,h}} u_h\| $ by an inequality on resolvents from Theorem~1 in~\cite{MR0220105}.

\begin{lemma}\label{lm:numerical}
With probability at least $1-\delta$, if $\sqrt{n}h^{4d+3} \geq \max \left( \left( \frac{3}{C_{\delta}^4}+1 \right) \frac{4C_1}{C_{\sigma}}, C_2 \right)$, $h\leq h_0$, then 
\begin{equation}\label{lm:numerical1}
\|u_h - \mathrm{Pr}_{u_{n,h}} u_h\| \leq \frac{8\pi}{C_{\sigma} h^2} \left( \|(T_{n,h} - T_h) u_h \| + \|(T_h - T_{n,h}) T_{n,h}\| \right).
\end{equation}
\end{lemma}

\begin{proof}

Applying to $\lambda_h \in \sigma(T_h)$ and $\Gamma_{r}$ with $r = C_{\sigma} h^2$, Lemma~\ref{lm:bd2} implies that
\begin{equation}\label{lm:numerical2}
\dist(\Gamma_{r}, 0) \geq \frac{1}{2} \text{ and } \max_{z\in \Gamma_{r}} \| R_z( T_h )\| \leq \frac{1}{C_{\sigma} h^2}.
\end{equation}

From Theorem~1 in~\cite{MR0220105}, we have that for $z\in \Gamma_{r}$, $u\in C(M)$ with $\|u\|=1$,
\begin{equation}\label{lm:numerical3}
\begin{aligned}
\displaystyle \|R_z(T_h) u - R_z(T_{n,h}) u\| &\leq \|R_z(T_h)\| \frac{\|(T_h - T_{n,h}) u\| + \|R_z(T_h) u\| \|(T_h - T_{n,h}) T_{n,h}\|}{\lvert z \rvert - \| R_z(T_h) \| \| (T_h - T_{n,h}) T_{n,h} \|} \\
& \leq \frac{2\|R_z(T_h)\|}{\lvert z \rvert}\left(\|(T_h - T_{n,h}) u\| + \|R_z(T_h)\|\|(T_h - T_{n,h}) T_{n,h}\|\right) \\
& \leq \frac{4}{C_{\sigma}^2 h^4}\left(\|(T_h - T_{n,h}) u\| + \|(T_h - T_{n,h}) T_{n,h}\|\right),
\end{aligned}
\end{equation}
where the second inequality assumes that
$$
\|(T_h - T_{n,h}) T_{n,h}\| \leq \frac{\lvert z \rvert}{2\|R_z(T_h)\|},
$$
which holds when
\begin{equation}\label{lm:numerical4}
\|(T_h - T_{n,h}) T_{n,h}\| \leq \frac{C_{\sigma}h^2}{4}.
\end{equation}
Therefore,
\begin{equation}
\begin{aligned}
\| u_h - \mathrm{Pr}_{u_{n,h}} u_h \| &= \| \mathrm{Pr}_{u_{h}} u_h - \mathrm{Pr}_{u_{n,h}} u_h \| = \| \int_{z\in \Gamma_{r}}R_z(T_h) u_h dz -  \int_{z\in \Gamma_{r}}R_z(T_{n,h}) u_h dz \| \\
& \leq \int_{z\in \Gamma_{r}} \| R_z(T_h) u_h -  R_z(T_{n,h}) u_h \| \lvert dz\rvert\\
& \leq \frac{8\pi }{C_{\sigma} h^2}\left(\|(T_h - T_{n,h}) u_h\| + \|(T_h - T_{n,h}) T_{n,h}\|\right)
\end{aligned}
\end{equation}
 
By Lemma~\ref{lm:empbd} and~\ref{lm:normT}, the condition~\eqref{lm:numerical4} holds when
$$
\left( \frac{3}{C_{\delta}^4}+1 \right) h^{-4d} \frac{C_1}{\sqrt{n}h} \leq \frac{C_{\sigma}h^2}{4}, \quad \sqrt{n}h^{d+1}\geq C_2.
$$
This holds when $\sqrt{n}h^{4d+3} \geq \max \left( \left( \frac{3}{C_{\delta}^4}+1 \right) \frac{4C_1}{C_{\sigma}}, C_2 \right)$.
\end{proof}

\subsection{Step III: Convergence Rate of Eigenfunctions}

We are now ready to show the convergence of eigenfunctions. By the triangle inequality and Lemma~\ref{lm:nT} and~\ref{lm:empbd},
\begin{equation}\label{thm:convvec}
\begin{aligned}
\|(T_{n,h} - T_h) u_h \| &\leq \| (\bar{T}_{n,h} - T_h) u_h \| +\| T_{n,h} - \bar{T}_{n,h} \| \\
 & \leq \left(\frac{1}{C_{\delta}^3}h^{-3d} +1 \right) \sup_{f\in \FF_h} |P_n f - P f| \\
& \leq \left(\frac{1}{C_{\delta}^3}h^{-3d} +1 \right) \frac{C_1}{\sqrt{n}h}.
\end{aligned}
\end{equation}

Proposition~18 in~\cite{MR2396807} states that
$$
\|a_n u_{n,h} - u_h\| \leq 2\| u_h - \mathrm{Pr}_{u_{n,h}} u_h \|.
$$
Then we use Lemma~\ref{lm:numerical} to bound the right-hand side to get
\begin{equation}\label{thm:convvec1}
\begin{aligned}
\|a_n u_{n,h} - u_h\| &\leq  \frac{8\pi}{C_{\sigma} h^2} \left( \frac{1}{C_{\delta}^3}h^{-3d} +1 + \left( \frac{3}{C_{\delta}^4}+1 \right)h^{-4d} \right) \left( \frac{C_1}{\sqrt{n}h} \right) \\
&\leq \frac{8\pi C_1}{C_{\sigma}} \left( \frac{1}{C_{\delta}^3}  + \frac{3}{C_{\delta}^4}+2  \right) \frac{1}{\sqrt{n}h^{4d+3}} = \frac{C''}{\sqrt{n}h^{4d+3}},
\end{aligned}
\end{equation}
where $C'' = \frac{8\pi C_1}{C_{\sigma}} \left( \frac{1}{C_{\delta}^3}  + \frac{3}{C_{\delta}^4}+2  \right) $ and $n,h$ satisfy the condition in Lemma~\ref{lm:numerical}. If we require $\sqrt{n}h^{4d+3}\rightarrow \infty$, the right-hand side of~\eqref{thm:convvec1} converges to 0 with probability at least $1-\delta$.

\subsection{Step IV: Convergence Rate of Eigenvalues}

We now deduce the convergence rate of eigenvalues from the convergence rate of eigenfunctions. 

\begin{equation}\label{eq:nlbd}
\begin{aligned}
|\lambda_h - \lambda_{n,h}| &=\|\lambda_h u_h - \lambda_{n,h} u_h\| \leq \|\lambda_h u_h - \lambda_{n,h} a_n u_{n,h}\| + |\lambda_{n,h}|\|u_h - a_n u_{n,h}\| \\
& \leq \|T_h u_h - a_n T_{n,h} u_{n,h}\| + |\lambda_{n,h}|\|u_h - a_n u_{n,h}\| \\
& \leq \|T_h u_h - T_{n,h} u_h\| + \|T_{n,h}\| \|(u_h - a_n u_{n,h})\| + |\lambda_{n,h}|\|u_h - a_n u_{n,h}\| \\
& \leq \left(\frac{1}{C_{\delta}^3} +1 \right) \frac{C_1}{\sqrt{n}h^{3d+1}} + \left(\frac{2}{C_{\delta}h^d} \right)\frac{C''}{\sqrt{n}h^{4d+3}}\\
& \leq \left( \frac{C_1}{C_{\delta}^3} +C_1 +\frac{2C''}{C_{\delta}}\right) \frac{1}{\sqrt{n}h^{5d+3}}= \frac{C'}{\sqrt{n}h^{4d+3}},
\end{aligned}
\end{equation}
where $C' = \frac{C_1}{C_{\delta}^3} +C_1 +\frac{2C''}{C_{\delta}}$. Here the first three inequalities follow from the triangle inequality. The fourth inequality follows from~\eqref{lm:dg1}, \eqref{thm:convvec} and~\eqref{thm:convvec1} and the fact that $|\lambda_{n,h}|\| \leq \|T_{n,h}\|$. As long as $\sqrt{n}h^{5d+3}\rightarrow \infty$, we obtain the convergence of eigenvalues with probability at least $1-\delta$.

\section{Proof of Corollary~\ref{cor:cr}: Convergence to $\Delta_M$}\label{sec:cor:cr}

When $P$ is the uniform distribution on $M$, Theorem~4.3 and Proposition~4.4 in~\cite{DBLP:conf/nips/BelkinN06} imply that there is a constant $C_l>0$ such that
\begin{equation}\label{eq:lbd}
\| \frac{1-\lambda_h}{h^2} - \lambda_i\| \leq C_l h^{\frac{4}{d+6}}.
\end{equation}
Applying Davis-Kahn theorem to $T_h$ and $\Delta_M$ (the Laplacian operator), it is easy to obtain the bound on eigenfunctions
\begin{equation}\label{eq:hbd}
\| u_h - u_i\| \leq C_u h^{\frac{4}{d+6}}.
\end{equation}

Combining~\eqref{eq:lbd} and~\eqref{eq:nlbd}, we get
$$
|\frac{1-\lambda_{n,h}}{h^2} -\lambda_i| \leq \frac{C'}{\sqrt{n}h^{5d+5}} + C_l h^{\frac{4}{d+6}}.
$$
Combining~\eqref{eq:hbd} and~\eqref{thm:convvec1}, we get
$$
\|a_n u_{n,h} - u_i\| \leq \frac{C''}{\sqrt{n}h^{4d+3}} +  C_u h^{\frac{4}{d+6}}.
$$
If we pick $\frac{1}{\sqrt{n}h^{5d+5}} = h$ (or $h = n^{-1/(10d+12)}$), then both the eigenvalue and eigenfunction converge with a rate of $O(h^{\frac{4}{d+6}}) = O(n^{-2/(5d+6)(d+6)})$.

\section{Proof of Corollary~\ref{cor:cs}: Consistency of Spectral Clustering}

We now show the consistency of spectral clustering in the two-region model. We first fix the notation. Suppose the data set $X = \{x_i\}_{i=1}^n$ is ordered such that the first $n_1$ points belong to $M_1$, the second $n_2$ points belong to $M_2$ where $n_1+n_2=n$. The graph Laplacian can be decomposed as
$$
\frac{I - \hat{T}_{n,h}^{u}}{h^2} = \left[ \begin{array}{cc}
A_{P_1,h} & B_h \\
B_h^T & A_{P_2,h} \end{array} \right] = \left[ \begin{array}{cc}
A_{P_1,h} & \mb{0} \\
\mb{0} & A_{P_2,h} \end{array} \right]+\left[ \begin{array}{cc}
\mb{0} & B_h \\
B_h^T & \mb{0} \end{array} \right].
$$
The degree matrix $D_{n,h}$ has a block decomposition $D_{11,h} \oplus D_{22,h}$. Let $\hat{T}_{n_i,h}^u$ and $D_{n_i,h}$ denote the degree matrix computed within each cluster for $i=1,2$ respectively. We note that $D_{ii,h} \not= D_{n_i,h}$.

Since $\dist(M_1, M_2) \geq c_d$, we have
$$
\max_{i,j}(B_h)_{ij} \leq \exp\{-c_d^2/h^2\}.
$$ 
This implies that entries of $D_{ii,h} - D_{n_i,h}$ never exceed $n \exp\{-c_d^2/h^2\}$. Moreover, 
$$
A_{P_i,h} = \frac{1}{h^2} (I - D_{ii,h}^{-1}D_{n_i, h} \hat{T}_{n_i,h}^{u}) = \frac{I - \hat{T}_{n_i,h}^{u}}{h^2} + \frac{1}{h^2} D_{ii,h}^{-1} (D_{ii,h} - D_{n_i,h}) \hat{T}_{n_i,h}^{u}.
$$
Let $E_i = \frac{1}{h^2} D_{ii,h}^{-1} (D_{ii,h} - D_{n_i,h}) \hat{T}_{n_i,h}^{u}$ and $F_i = \frac{I - \hat{T}_{n_i,h}^{u}}{h^2}$. We note that
$$
\frac{I - \hat{T}_{n,h}^{u}}{h^2} = F_1 \oplus F_2 + E_1\oplus E_2 +\left[ \begin{array}{cc}
 \mb{0} & B_h \\
 B_h^T & \mb{0} \end{array} \right].
$$
With probability at least $1-\delta$,
\begin{equation}\label{eq:pert}
\begin{aligned}
\|E_i\| &\leq  \frac{1}{h^2}\|D_{ii,h}^{-1}\| \|D_{ii,h} - D_{n_i,h}\| \|\hat{T}_{n_i,h}^{u}\| \\
& \leq \frac{1}{h^2} \frac{1}{C_{\delta}h^d}    n \exp\{-c_d^2/h^2\} \frac{1}{C_{\delta}h^d} = \frac{n}{C_{\delta}^2 h^{2d+2}} \exp\{-c_d^2/h^2\}.
\end{aligned}
\end{equation}
The second inequality follows from the bounds of $d_{n,h}$ and $\|\hat{T}_{n_i,h}^{s}\|$ in Lemma~\ref{lm:dg} by noting that $\|D_{ii,h}^{-1}\|$ is diagonal and $\|\hat{T}_{n_i,h}^{u}\|$ has the same eigenvalues of $\|\hat{T}_{n_i,h}^{s}\|$ (assuming $\sqrt{n_i}h^{d+1}\geq C_2$). 

Moreover, we have
\begin{equation}\label{eq:pert1}
\begin{aligned}
\left\|\left[ \begin{array}{cc}
 \mb{0} & B_h \\
 B_h^T & \mb{0} \end{array} \right]\right\| & \leq n \max_{i,j}(B_h)_{ij} \leq n \exp\{-c_d^2/h^2\}.
\end{aligned}
\end{equation}

Let $0 = \lambda_{0} = \lambda_{1} < \lambda_2 \leq \cdots$ be the eigenvalues of $\Delta_{P_1}\oplus\Delta_{P_2}$ and $\lambda_{0,h}^n < \lambda_{1,h}^n \leq \cdots$ be those of $F_1 \oplus F_2$. Theorem~\ref{thm:cr} states that $\lambda_{i,h}^n$ converges to $\lambda_i$ as $h\rightarrow 0$ and $\sqrt{n}h^{5d+3}\rightarrow \infty$. Therefore there exists $N_0, h_0 > 0$ such that when $h<h_0$ and $n>N_0$,
$$
\dist(\{\lambda_{0,h}^n\} \cup \{\lambda_{1,h}^n\}, \cup_{i>2}\{\lambda_{i,h}^n\} ) \geq \frac{1}{2}\lambda_2.
$$
On the other hand,~\eqref{eq:pert} and~\eqref{eq:pert1} imply that
$$
\|\frac{I - \hat{T}_{n,h}^{u}}{h^2} - F_1 \oplus F_2 \| \leq \frac{n}{C_{\delta}^2 h^{2d+2}} \exp\{-c_d^2/h^2\} +n \exp\{-c_d^2/h^2\}.
$$

This means that the norm of the perturbation of the graph Laplacian from $F_1 \oplus F_2$ is much smaller than $\frac{1}{2} \lambda_2$, the gap lower bound between the 2nd and 3rd eigenvalues of $F_1 \oplus F_2$ when $h\rightarrow 0$, $\sqrt{n}h^{5d+3}\rightarrow \infty$. By Davis-Kahn theorem, the first two eigenfunctions of $\frac{I - \hat{T}_{n,h}^{u}}{h^2}$ or $\frac{I - \hat{T}_{n,h}}{h^2}$ converge to those of $\Delta_{P_1}\oplus\Delta_{P_2}$, (e.g. the indicator functions $\mb{1}_{M_1}$ and $\mb{1}_{M_2}$). The convergence of eigenfunctions implies the convergence of the eigen-mapping. Then the consistency of the spectral clustering is obtained from the stability of $K$-means clustering on the mapped dataset (see Theorem~1.8 of~\cite{Nicolas15}).

\section{A Numerical Study: The Noisy Case}

Let $T_{\sigma} M = \{x | \dist(x, M) \leq \sigma\}$ be the tubular neighborhood of $M$ with a probability $P_{T_{\sigma} M}$. When the dataset $X$ is i.i.d.~sampled from the tubular neighborhood of $M$, we define the function spaces $\mathcal{K}_h$, $\HH_h$, $u\HH_h$, $\HH_h \HH_h$ and $\FF_h$ over $T_{\sigma} M$ similarly. Since $T_{\sigma} M$ is $D$ dimensional, we obtain the spectral convergence when $\sqrt{n}h^{5D+3}\rightarrow \infty$ by the same analysis. The convergence rate is hugely different from that in Theorem~\ref{thm:cr} when $d \ll D$. We use a very simple experiment to demonstrate how the dimension affects the quality of graphs. 

Although the theoretical analysis assumes every entry $A_{ij}=k_h(x_i, x_j)$ of the similarity matrix $A$ is non-zero, a thresholding ($r$-neighbor or $k$-neighbor) is applied to $A$ in practice. That is, $A_{ij}$ is non-zero only when $x_i$ is in the $r$-neighbor (or $k$-neighbor) of $x_j$, or vice versa. Zelnik and Perona~\cite{DBLP:conf/nips/Zelnik-ManorP04} advocate choosing $h$ to be the distance to the $k$-th neighbor. This is indeed a soft version of the $k$-neighbor construction. In general, there is lack of discussion in the literature w.r.t. the effect of different bandwidth choices. 

We note that the $r$-neighbor thresholding is equivalent to directly thresholding the value $k_h(x_i,x_j)$ when the bandwidth $h$ is fixed. To better observe the effect when $h$ varies, we take the natural strategy of directly thresholding the entry value $A_{ij}$, that is, setting $A_{ij}=0$ if $k_h(x_i, x_j)\leq c$ for a small constant $c$. The goal of the numerical experiment is to investigate the connectivity of the graph under this construction. Indeed, the graph of an unknown manifold to be connected is a necessary condition to tasks such as obtaining low-dimensional coordinates via algorithms such as Laplacian eigenmaps.

\subsection{Noise-free case}

The sample dataset $X=\{x_i\}_{i=1}^n$ is i.i.d.~sampled uniformly from the unit circle $\Sb^1$ in $\R^D$. The entry value $A_{ij}=0$ if $k_h(x_i, x_j)\leq 0.001$. Given $n$ and $D$, we generate $K=300$ sample datasets. For each sample dataset, the minimal value of $h$ that ensures a connected graph is found. Then the mean of $h$ across $K$ samples is reported in Figure~\ref{fig:noisefree}. It is obvious that the connectivity is independent of the ambient dimension $D$ in the noisefree case.

\begin{figure*}[htb!]
\centering
\includegraphics[width=0.8\textwidth]{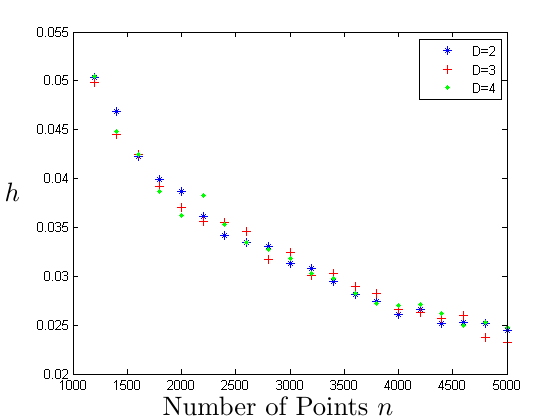}
\caption{mean value of $h$ that generates a connected graph without noise.}
\label{fig:noisefree}
\end{figure*}

\subsection{Noisy case}

The sample dataset $\{x_i\}_{i=1}^n$ is i.i.d.~sampled uniformly from the unit circle $\Sb^2$ in $\R^D$. The dataset $X=\{\hat{x}_i\}_{i=1}^n$ is created by picking $\hat{x}_i$ uniformly from the $r$-ball $B(x_i,r)$. We note that $X$ is sampled from a distribution very similar (not exactly equal) to the uniform distribution of the tubular neighborhood $T_r(\Sb^1)$. In this experiment, we use $r=0.1$. The entry value $A_{ij}=0$ if $k_h(x_i, x_j)\leq 0.001$. Given $n$ and $D$, we generate $K=300$ sample datasets. For each sample dataset, the minimal value of $h$ that ensures a connected graph is found. Then the mean of $h$ across $K$ samples is reported in Figure~\ref{fig:noisy}. We note that $h$ has to be picked larger to ensure connectivity as the ambient dimension $D$ increases. 

\begin{figure*}[htb!]
\centering
\includegraphics[width=0.8\textwidth]{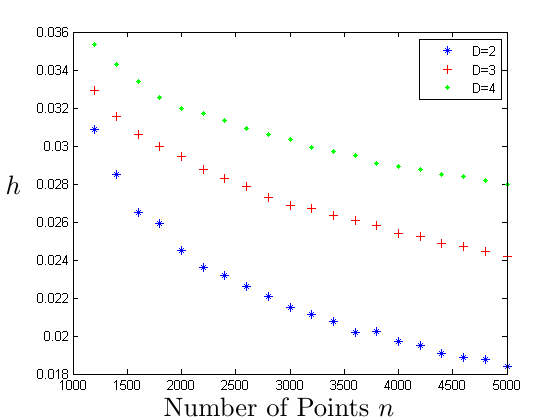}
\caption{mean value of $h$ that generates a connected graph with noise.}
\label{fig:noisy}
\end{figure*}

\section{Conclusion}

This work establishes the spectral convergence rate of graph Laplacians to the underlying Laplacian operator on $M$. Through a preliminary numerical study, we also emphasize the ``difficulty'' of applying spectral algorithms in high dimension and the necessity of a denoising step. We also provide a list of open questions for future research. The first question is about the quadratical dependence on $d$ of the convergence rate. It is unclear whether a linear dependence can be achieved, which we believe is related to the convergence from $T_h$ to $\Delta_M$. For the spectral clustering algorithm, the convergence rate is still unknown since the convergence rate analysis of the $K$-mean step is missing. We also show theoretically and empirically that the convergence rate of the graph Laplacian depends on the ambient dimension $D$ once noise is considered. It is interesting to develop denoising techniques for spectral methods, that can yield a convergence rate independent of the ambient dimension $D$.

{\small
\bibliography{biblio}{}
\bibliographystyle{plain}
}

\end{document}